\documentclass{article}

\usepackage{microtype}
\usepackage{graphicx}
\usepackage{subfigure}
\usepackage{booktabs} % for professional tables
\usepackage{array}
\usepackage{xurl}
\usepackage[most]{tcolorbox}
\usepackage{wrapfig}
\usepackage{enumitem} %for leftmargin
\usepackage{tabularx} % For adjustable-width columns
\usepackage{caption}
\usepackage{longtable}
\usepackage{tabularx}
\usepackage{ltablex}
\usepackage{multirow}
\usepackage{mathtools}
\usepackage[table]{xcolor}
\usepackage{makecell}
\usepackage{algorithm}
\usepackage{algpseudocode}
\usepackage{bbm}
\usepackage{amsmath,amsthm,amssymb}
\usepackage{framed}

\usepackage{hyperref}
\providecommand{\doi}[1]{\href{https://doi.org/#1}{doi: \nolinkurl{#1}}}

\newcommand{\authorcomment}[3]{%
}
\usepackage[accepted]{icml2025}

\usepackage[capitalize,noabbrev]{cleveref}

\theoremstyle{plain}
\newtheorem{theorem}{Theorem}[section]
\newtheorem{proposition}[theorem]{Proposition}
\newtheorem{lemma}{Lemma}[theorem]

\theoremstyle{definition}

\newtheorem{assumption}[theorem]{Assumption}
\crefname{assumption}{Assumption}{Assumptions}
\Crefname{assumption}{Assumption}{Assumptions}
\theoremstyle{remark}
\newtheorem{remark}[theorem]{Remark}

\usepackage[disable,textsize=tiny]{todonotes}
\icmltitlerunning{Is Code Better Than Language for Algorithmic Reasoning?}

\providecommand{\E}{\mathbb{E}}

\providecommand{\R}{\mathbb{R}}
\providecommand{\cH}{\mathcal{H}}

\setlist{nosep,leftmargin=*,topsep=0pt,partopsep=0pt}
\begin{document}

\twocolumn[
\icmltitle{Is Code Better Than Language for Algorithmic Reasoning?}


\icmlsetsymbol{equal}{*}

\begin{icmlauthorlist}
\icmlauthor{Terry Tong}{yyy}
\icmlauthor{Yu Feng}{yyy}
\icmlauthor{Surbhi Goel}{yyy}
\icmlauthor{Dan Roth}{yyy,oracle}
\end{icmlauthorlist}

\icmlaffiliation{yyy}{University of Pennsylvania}
\icmlaffiliation{oracle}{Oracle AI}

\icmlcorrespondingauthor{Terry Tong}{tongt1@seas.upenn.edu}

\icmlkeywords{Neuro-symbolic AI, Tool Use, Bayesian Inference, Algorithmic Reasoning, Code Generation}

\vskip 0.3in
]



\printAffiliationsAndNotice{\icmlEqualContribution} % otherwise use the standard text.

\begin{abstract}
% Context: Establish the problem domain

For tool-augmented language models, comparing natural-language reasoning with code-execution pipelines is difficult because the comparison changes both the intermediate representation and the execution mechanism. We separate these factors with an intermediate intervention: the model expresses its reasoning as executable code, and the language model simulates that code in context to produce an answer. On a 40-task verifiable algorithmic benchmark, deterministic code execution outperforms natural-language reasoning by +31.6pp. We observe that the intermediate intervention is not meaningfully different from natural-language reasoning (+0.15pp). These results suggest that, in our evaluated setting, changing the intermediate representation alone does not explain the tool-use advantage, providing evidence for the performance gains requiring reliable external execution.  We formalize this intuition with a simple statistical decision-theoretic model that characterizes when execution dominates end-to-end risk in our disentangled trace-generation/execution regime. We validate our theory using a reconstruction intervention that leverages a proxy language model to infer natural-language reasoning traces from code representations, recovering performance comparable to the original natural-language reasoning pipeline. All experiments are at \url{https://github.com/TerryTong-Git/ToolProj}.

\end{abstract}

\section{Introduction} \label{sec:intro}

Many agentic systems orchestrate symbolic solvers, LLMs, and other tools, achieving state-of-the-art performance~\cite{gao2023pal,yao2023react, 10.5555/3666122.3669119, yang2024sweagent,wang2025mcpbenchbenchmarkingtoolusingllm}. Prior work shows that translating problems into solver-executable code (Route~3) and delegating execution often outperforms end-to-end NL reasoning (Route~1) on logic- and algorithmic-style complex reasoning tasks~\cite{lyu2023faithful, pan2023logic}.

\smallskip
It is unclear whether these gains come from the structured code representation, the reliability of an external executor, or both. A direct comparison is ill-posed because the routes learn different objects: natural-language traces versus solver-executable programs. Without a common intermediate route, performance gaps conflate representation and execution.

\begin{figure}[t]
    \centering
    \vspace{-20pt}
    \includegraphics[width=\linewidth]{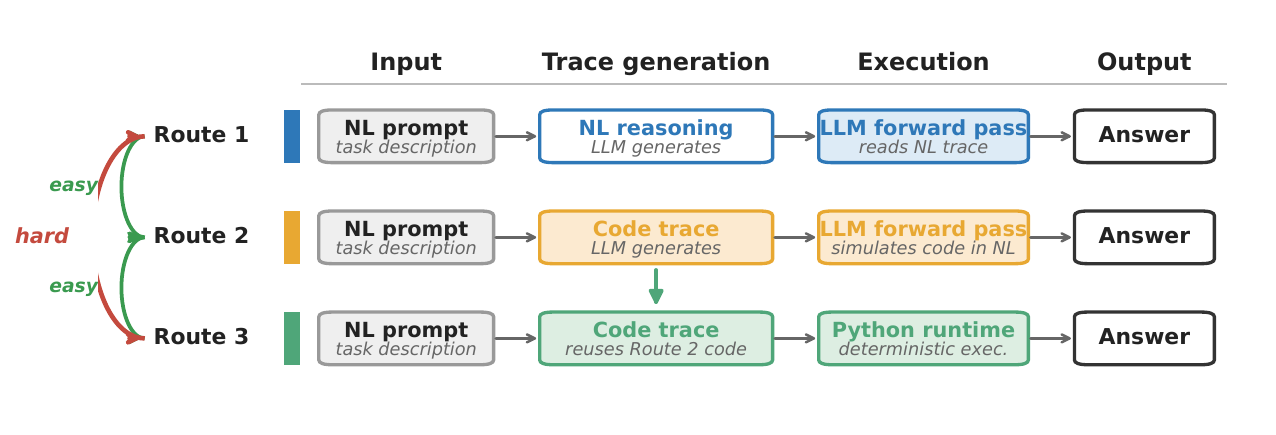}
    \vspace{-22pt}
    \caption{\textbf{Three-Route Framework.} We decompose algorithmic reasoning into: (1) \emph{Translation} into NL or code, and (2) \emph{Execution} via LLM or solver. This yields three routes: \textbf{Route~1} (Direct NL), \textbf{Route~2} (Code + NL simulation), \textbf{Route~3} (Code + Solver Execution). Prior work compares only Route~1 vs.\ Route~3, confounding translation and execution. Our Route~2 isolates these factors. }
    \label{fig:three_routes}
\end{figure}

\begin{figure*}[t]
    \centering
    \vspace{-5pt}
    \includegraphics[width=\textwidth]{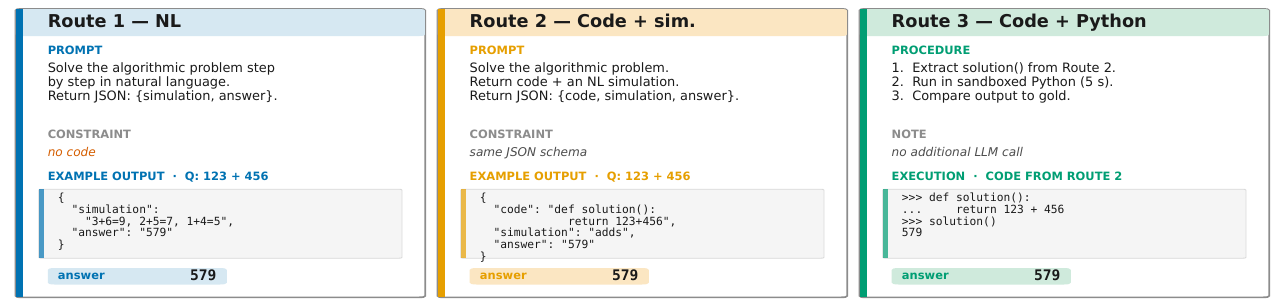}
    \vspace{-28pt}
    \caption{\textbf{Prompt templates for three-route evaluation.} \textbf{Route~1 (NL)}: LLM reasons in natural language only, code forbidden. \textbf{Route~2 (Sim)}: LLM generates Python \texttt{solution()} then simulates execution in NL. \textbf{Route~3 (Code)}: Same code executed in Python runtime. This isolates execution mechanism while controlling translation.}
    \label{fig:route_prompts}
\end{figure*}

This paper presents a systematic three-route framework (\cref{fig:three_routes,sec:framework}) that makes the comparison tractable. We decouple the pipeline into trace generation (e.g. Chain of Thought \cite{wei2022chain}) and execution (see \cref{def:trace}), instantiated by controlled prompts in \cref{fig:route_prompts}. Route~2 fixes the executor class while changing the representation through code generation followed by LLM simulation. Route~3 reuses the same code trace and delegates execution to Python.

\smallskip
Empirically, the framework yields Route~1 (NL + NL reasoning) $\approx$ Route~2 (code + NL simulation) $<$ Route~3 (code + Python execution) (\cref{sec:performance}). On the 40 task benchmark, the accuracies are 17.21\%, 17.37\%, and 48.84\%, respectively. The paired Route~2--Route~1 gap is +0.15pp with 95\% cluster-bootstrap CI $[-0.30,+0.61]$pp. Route~3 exceeds Route~2 by +31.47pp with CI $[+29.20,+33.71]$pp (\cref{fig:main,fig:per_task_accuracy}).

\smallskip
To test representation, we hold the LLM executor fixed and vary only the trace. Theory shows that code is risk non-inferior when natural language adds nuisance variation. The nuisance assumption seems to align with our intuition that a problem solution can be expressed in more ways in natural language than in code. A reconstruction experiment empirically demonstrates that code-derived NL traces also recover native-NL performance (\cref{sec:func_similarity,fig:translation_additivity}). Thus, code retains decision-relevant information, and representation is not the primary bottleneck in this setting.

% Because our goal is to optimize the end-to-end reasoning pipeline rather than any particular instantiation of it, we evaluate representations and executors in terms of their computation-constrained optimal risk (hereafter optimal risk). This criterion captures the best achievable performance within a model family and abstracts away suboptimal prompt or executor choices. We propose sufficient conditions such that the optimal risk of code is never worse than language up to a small amount of error $\varepsilon$. Furthermore, we empirically verify these conditions through interventions that show that code traces can be translated into natural language without losing decision-relevant information (\cref{sec:func_similarity}).

\smallskip
We then analyze Route~2 vs.\ Route~3 (\cref{sec:route2v3}), where the code trace is fixed and only the executor changes. The recovery mass where Route~2 succeeds while Route~3 fails is 1.61\%, and the execution-win mass where Route~3 succeeds while Route~2 fails is 33.08\% (\cref{fig:recovery_final}). As task difficulty increases, code execution maintains substantially higher accuracy. We therefore attribute the gap in our evaluated setting to reliable execution rather than code trace generation.

Our main contributions are:
\begin{enumerate}
    \item A \textbf{three-route framework} (\cref{sec:framework,fig:three_routes,fig:route_prompts}) for tractable comparison between code and natural language representations via an intermediary (code generation with LLM execution).
    \item \textbf{Empirical validation} (\cref{sec:empirical,fig:main,fig:per_task_accuracy}) demonstrating that Route~1 and Route~2 are statistically close, while Route~3 is substantially better.
    \item A \textbf{systematic study} (\cref{sec:rep_analysis,sec:route2v3}) ruling out trace generation as the bottleneck and solidifying execution as the bottleneck in end-to-end algorithmic reasoning performance when using language.
\end{enumerate}

\begin{figure*}[t]
    \centering
    \vspace{-5pt}
    \includegraphics[width=\textwidth]{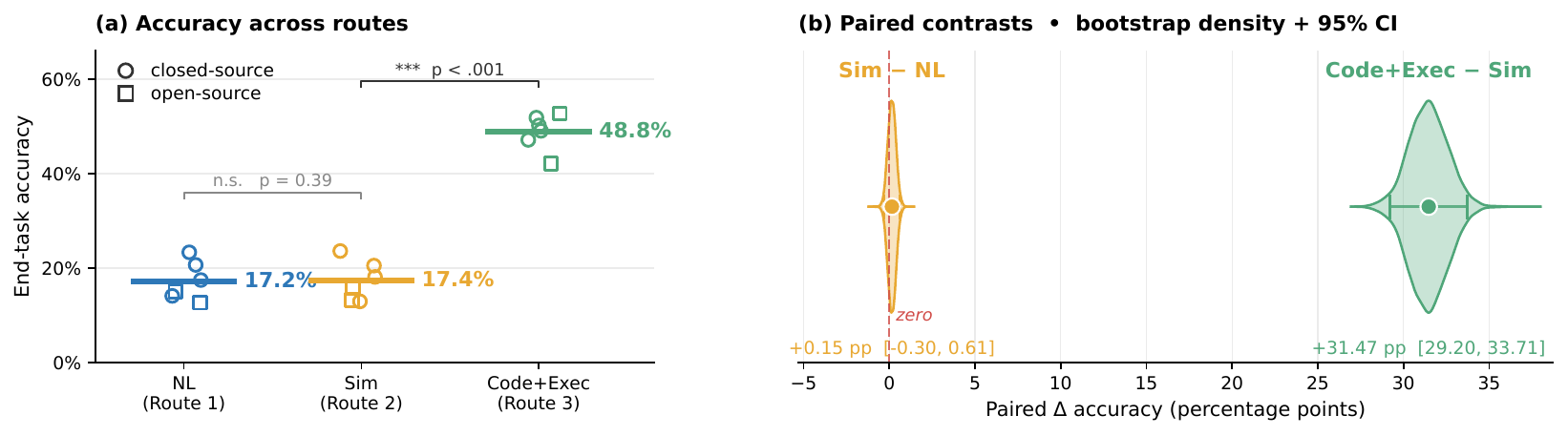}
    \vspace{-22pt}
    \caption{\textbf{Code + Solver Execution performs better than Direct NL reasoning quantified by end-task accuracy overall and within paired instances.} Paired instance contrasts are shown in the bootstrap distributions. Results are averaged over the 40-task analysis set with 1,113 unique problem instances, 3 seeds, and 6 models. }
    \label{fig:main}
\end{figure*}

\section{Three-Route Framework}
\label{sec:framework}
Our central research question (RQ) is: \emph{Is code $>$ NL for algorithmic reasoning?} To begin to answer this question, we first introduce our three-route framework which enables tractable comparison by disentangling \emph{reasoning representation} (hereafter called Traces) and \emph{reasoning execution} (hereafter called Executors) and constructing an intermediary bridge (Route 2) for pairwise comparison. This section introduces the task (\cref{def:task}), notation (\cref{def:trace}), and route definitions (\cref{def:route}).

% \yu{It's not only information representation, reasoning modality/formalization might be a better word, since it includes both representation and intermediate reasoning} from differences in \emph{execution}, enabling tractable and interpretable comparisons. \yu{We first present basic definitions in the paper through \S\ref{def:task} to \S\ref{def:risk} and then introduce the three-arms framework in \S\ref{def:arm}. Note: always try to conclude/foresee the section in the first paragraph} 

\subsection{Task and Loss}
\label{def:task}
% \yu{Let $p(x)$ be the test distribution over task instances, and let $X \sim p(x)$ denote a sampled instance (problem statement and inputs). Q: test distribution across all possible algorithmic tasks or just one?}
For a gold evaluation distribution over task instances $i$ specified by (algorithm, input variables, seed), let $X \sim p(x)$ denote a task instance (problem statement and inputs), and let $Y^*(X) \in \mathcal{Y}$ denote the ground-truth output that is unique, fixed, and externally verifiable. 
% \yu{, where 
% $\mathcal{Y}$ is the set of possible outputs. Q: is the output fixed/unique/verifiable, if so, add unique, fixed, and externally verifiable ground-truth output}
We evaluate performance under $0$--$1$ loss:
$$\ell(y,x) := \mathbf{1}\{y \neq Y^*(x)\}.$$
% All risks are taken with respect to the evaluation distribution $p(x)$.

\subsection{Trace Generators and Executors}
\label{def:trace}
We model each reasoning pipeline as a two-stage stochastic process with \emph{(1) trace generation} and \emph{(2) execution}. 
\paragraph{Trace generation.} We define a \emph{trace} as an object that stores intermediate reasoning used to solve a corresponding task (e.g. NL Chain-of-Thought or a program). A \emph{trace generator} is a (stochastic) mapping
\[
E : \mathcal{X} \to \Delta(\mathcal{Z}), \qquad
Z \sim p_E(z \mid x),
\]
% \yu{$p_E$ looks a bit wierd}
which produces an auxiliary trace $Z$ given the task instance.

\paragraph{Execution.} We define \emph{execution} as a procedure that consumes a trace, e.g. an LLM forward pass that takes as input a reasoning trace and outputs an answer, or executing a program in an external runtime. An \emph{executor} is a (stochastic) mapping
\[
\rho : \mathcal{X} \times \mathcal{Z} \to \Delta(\mathcal{Y}),
\]
mapping the observed instance and trace to a final output.

% The induced conditional output distribution is
% \[
% P_{\rho,E}(Y = y \mid X = x)
% = \int \rho(y \mid x, z)\, p_E(z \mid x)\, dz .
% \]

% The population risk of a pipeline $(E,\rho)$ is
% \begin{align*}
% & R(E,\rho)
% := \mathbb{E}\bigl[\ell(\hat{Y}, X)\bigr],
% \qquad \\
% & Z \sim p_E(\cdot \mid X),\;
% \hat{Y} \sim \rho(\cdot \mid X, Z).
% \end{align*}

% \subsection{Computation-Constrained Optimal Risk}
% \label{def:risk}
% To reflect inference-time computational constraints, we evaluate executors within restricted families.
% For a trace space $\mathcal{Z}$ (e.g. code or NL), let
% \[
% \mathcal{H}_{\mathcal{Z}} \subseteq
% \{ \rho : \mathcal{X} \times \mathcal{Z} \to \Delta(\mathcal{Y}) \}
% \]
% denote a class of executors realizable under a fixed model family, e.g. Mistral (see \cref{sec:cond1}).
% % and bounded inference protocol\yu{which will be introduce in \S\ref{} add a reference here}.
% We define the computation-constrained optimal risk of a trace generator $E$ as
% \[
% R^*_{\mathcal{H}}(E)
% := \inf_{\rho \in \mathcal{H}} R(E,\rho).
% \]
% This differs from classical Bayes risk, which optimizes over all measurable decision rules.

\subsection{The Three Routes}
\label{def:route}
We consider three reasoning pipelines (``routes''), illustrated in \cref{fig:three_routes}. Each route is represented as a pair $(E,\rho)$ consisting of a trace generator $E$ and an executor $\rho$.
% \yu{A concrete example of the three routes can be referred to in Figure~\ref{fig:main}.}
% \yu{I will add some NL descriptions which can be referred to in other sections. }
\paragraph{Route 1 (Direct Natural Language).}
Route~1 represents standard NL reasoning: the model first produces a natural-language (Chain-of-thought) trace and then the same model is used as the LLM-based executor, conditioning on the trace to generate a final answer. Formally, a natural-language trace generator $E_{\mathrm{NL}}$ produces traces
$Z_{\mathrm{NL}} \sim p_{\mathrm{NL}}(\cdot \mid X)$,
paired with an executor family
\[
\mathcal{H}_{\mathrm{NL}}
\subseteq
\{ \rho : \mathcal{X} \times \mathcal{Z}_{\mathrm{NL}} \to \Delta(\mathcal{Y}) \}.
\]
\paragraph{Route 2 (Code + NL Simulation).}
Route~2 uses \emph{code} as the trace modality, but keeps execution ``in-model'': the LLM instead simulates the generated code in natural language, rather than executing it in an external environment. This route isolates the effect of using a code trace while holding the LLM-based executor class fixed. Formally, a code trace generator $E_{\mathrm{Code}}$ produces executable representations
$Z_{\mathrm{C}} \sim p_{\mathrm{Code}}(\cdot \mid X)$,
paired with an executor family
\[
\mathcal{H}_{\mathrm{C}}
\subseteq
\{ \rho : \mathcal{X} \times \mathcal{Z}_{\mathrm{C}} \to \Delta(\mathcal{Y}) \}
\]
corresponding to language-model-based simulation of code execution.
\paragraph{Route 3 (Code + Solver Execution).}
Route~3 uses the same code trace generator $E_{\mathrm{Code}}$ as Route~2, producing
$Z_{\mathrm{C}} \sim p_{\mathrm{Code}}(\cdot \mid X)$,
but pairs it with a deterministic executor corresponding to external code execution.
Let $\mathrm{Exec} : \mathcal{X} \times \mathcal{Z}_{\mathrm{C}} \to \mathcal{Y}$ denote the (fixed) external runtime
that executes the code trace $z$ on instance $x$ and returns an output.
Our executor is
\begin{align*}
& \rho_{\mathrm{Exec}}(y \mid x,z) := \mathbf{1}\{y = \mathrm{Exec}(x,z)\}.
\end{align*}
The corresponding executor family is the singleton $\mathcal{H}_{\mathrm{Exec}} := \{\rho_{\mathrm{Exec}}\}.$

% \subsection{Interpreting the Three routes}

% In all arms, the executor observes the full task instance $X$; the trace $Z$ provides auxiliary information. Accordingly, comparisons between arms are conditional on $X$ and should be interpreted as statements about what auxiliary representations can be simulated from others, and how execution mechanisms affect achievable risk.

% Performance comparisons are expressed in terms of computation-constrained optimal risks $R^*_{\mathcal{H}}(E)$, which isolate differences due to representation and execution under fixed inference constraints.

% \subsection{Section outline}
% \label{roadmap}
% Section~\ref{sec:empirical} instantiates this framework with concrete models, prompts, datasets, and statistical tests. Section~\ref{sec:performance} provides empirical evidence motivating assumptions about the relationship between natural-language and code traces. Section~\ref{sec:theory} analyzes the three arms theoretically within this framework.

\begin{figure*}[t]
    \centering
    \vspace{-5pt}
    \includegraphics[width=\textwidth]{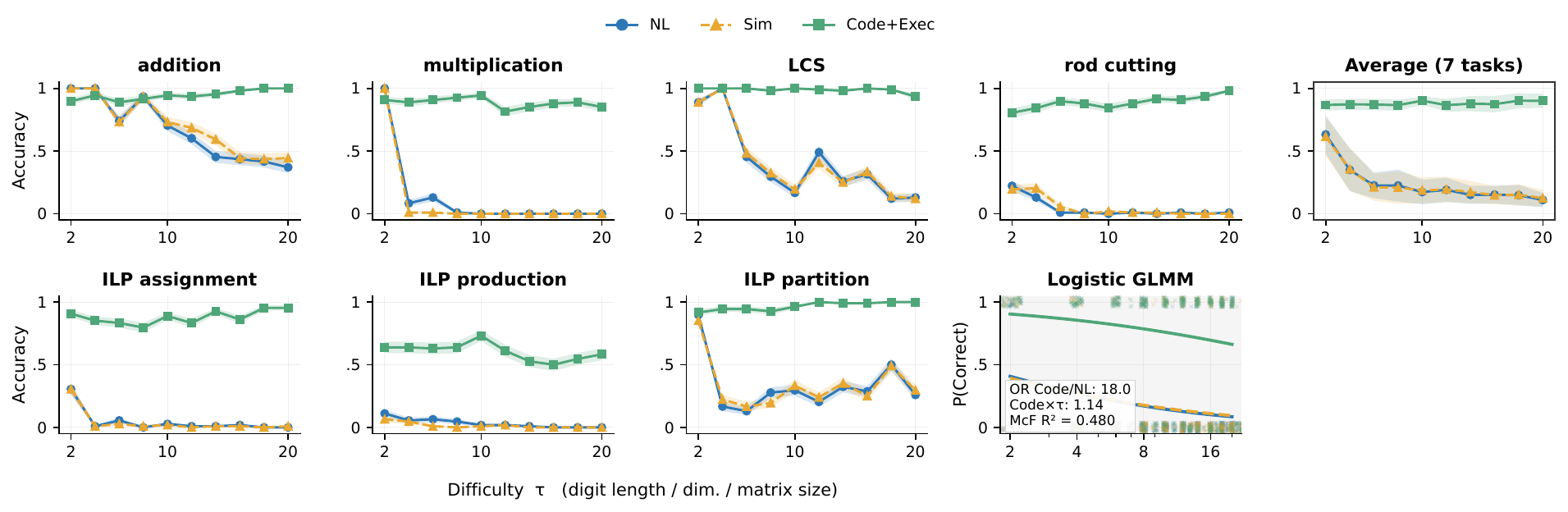}
    \vspace{-28pt}
    \caption{ \textbf{Code + Solver Execution scales better than natural language reasoning when problems get harder, both within-task and on average, interpolating and extrapolating (gray).} $\tau$ is used to control digit length for arithmetic, table dimensionality for dynamic programming, and constraint matrix dimensionality for integer linear programming. The figure uses the same data and setup as \cref{fig:main}. }
    \label{fig:per_task_accuracy}
\end{figure*}

\section{Evaluating the Three-Route Framework}
\label{sec:empirical}
Our evaluation aims to answer the research questions (RQ):
\begin{enumerate}
    \item[1)] Does $\text{Route~1} \approx \text{Route~2} < \text{Route~3}$ hold? (\cref{sec:instantiation,sec:performance})
\end{enumerate}
 % RQ 1 is central to the main thesis, and RQ 2 provides motivation to offload to an executor, as well as empirical foundations for the assumptions in  Section~\ref{sec:theory}.

\subsection{Experimental Instantiation of the Three Routes}
\label{sec:instantiation}

To draw conclusions about trace modality and execution method, we fix one stage at a time. For Route~1 vs.\ Route~2, we fix the execution phase by using the same LLM reasoning model forward pass, but use different modalities in the trace generator. For Route~2 vs.\ Route~3, we fix the trace generator that outputs code, then use different execution methods: either an LLM reasoning model forward pass or a Python~3 runtime. Below, let $X_i$ be the task instantiation of instance $i$. Structured JSON outputs are enforced for sampling in all routes
(\cref{fig:route_prompts}).

\paragraph{Sampling Route 1.}  Route 1 prompts an LLM $E_{NL} $ to function as a trace generator $Z_{NL} := E_{NL}(X_i)$. The trace is then fed into the same LLM $\rho_{NL} = E_{NL}$ to produce an answer $Y_{i}^{(NL)} := \rho_{NL} (Z_{NL})$. The prompt instructs the model to never use code, and to output a structured rationale and answer (see \cref{fig:route_prompts}). Operationalized, one experimental run for Route 1 is encapsulated in a single LLM reasoning model's forward pass without pause.
\paragraph{Sampling Route 2.} Similarly, Route 2 uses a LLM $E_{Code}$ as a trace generator for code modality $Z_{Code} := E_{Code}(X_i)$. The trace is then fed into the same LLM $\rho_{Sim} = E_{Code}$ to produce an answer $Y_{i}^{(Sim)} := \rho_{Sim} (Z_{Code})$. Operationalized, one experimental run for Route 2 is encapsulated in a single LLM reasoning model's forward pass without pause. Prompt templates are controlled to be as similar as possible, with
Route~2 differing only by the inclusion of a code-generation and
simulation instruction (\cref{fig:route_prompts}). We prompt the model to output a program snippet, structured reasoning over the code, followed by an answer.
\paragraph{Sampling Route 3.} Route 3 takes the exact same code trace $Z_{Code}$ as Route 2, except it runs it through a Python~3 runtime $\rho_{Exec}$ and retrieves the solution $Y_i^{(Exec)} = \rho_{Exec}(Z_{Code})$. The Python~3 runtime has access to five standard scientific libraries: \texttt{SciPy, NumPy, pandas, PuLP,} and \texttt{PyTorch}.
\paragraph{Observed outcomes.}
For each problem instance $i$, we observe paired Bernoulli outcomes \[\bigl(Y_i^{(\mathrm{NL})},\; Y_i^{(\mathrm{Sim})},\; Y_i^{(\mathrm{Exec})}\bigr).\]
% Let $X_i$ denote the original problem instance and $C_i$ the generated
% code for instance $i$. The tuple $(X_i, C_i)$ is held fixed across both
% arms. Arm~2 produces
% $(X_i, C_i) \;\mapsto\; Y_i^{(\mathrm{Sim})}$
% via language-model-based simulation, while Arm~3 executes the same code
% using an external Python runtime,
% $(X_i, C_i) \;\mapsto\; Y_i^{(\mathrm{Exec})}.$
% For notational symmetry, $X_i$ is included in both mappings, although it
% is ignored by the external executor.
\paragraph{Data and models.}
We evaluate a 40-task analysis set drawn from CLRS-30, NP-Hard-Eval,
and a custom fine-grained evaluation suite, across three random seeds
$\{0,1,2 \}$. The set contains 1,113 unique problem instances and
20,034 paired route evaluations after pooling six full-coverage models.
Tasks span arithmetic, dynamic programming, graph algorithms, string
algorithms, geometry, sorting, and NP-hard optimization, with difficulty
controlled by a parameter $\tau$ when applicable.
We evaluate closed-source models (\texttt{Claude Haiku~4.5, GPT-4o-mini,
Gemini~2.0 Flash, Gemini~2.5 Flash}) and open-source models
(\texttt{Mixtral-8x22b-Instruct, Codestral-2508}).
\Cref{app:complexity_stratified_route_accuracy,app:model_stratified_route_accuracy}
stratify these results by complexity and model, and
\cref{app:additional_model_evaluations} reports subset model coverage.

% Models with more than
% 50\% JSON parse failures are excluded to avoid confounding instruction
% following with reasoning ability.

% \paragraph{Prompting and execution.}
% In Arm~1, models are instructed to reason without using code and to
% output a structured rationale and answer.
% In Arm~2, the same prompt is augmented with instructions to generate
% code and simulate its execution.
% In Arm~3, the generated function is executed directly in a Python~3
% runtime with access to standard scientific libraries.

\subsection{End-to-End Performance Comparison}
\label{sec:performance}

We first define the paired tests used to compare the three routes, then report
the pooled route accuracies and paired differences, and finally summarize how
the gaps change with task difficulty.

\paragraph{Pairwise statistical tests.}
We evaluate pairwise route differences using McNemar tests on paired
Bernoulli outcomes.
For Route~1 vs.\ Route~2, the null hypothesis is
\begin{align*}
& H_0:\;
\Pr(Y^{(\mathrm{NL})}=1,\,Y^{(\mathrm{Sim})}=0) \\
& =
\Pr(Y^{(\mathrm{NL})}=0,\,Y^{(\mathrm{Sim})}=1),
\end{align*}
with an analogous null for Route~2 vs.\ Route~3.

Effect sizes are reported as paired accuracy differences, e.g.,
\[
\Delta_{\mathrm{Sim-NL}}
=
\mathrm{Acc}(\mathrm{Sim}) - \mathrm{Acc}(\mathrm{NL}),
\]
with 95\% confidence intervals estimated via cluster bootstrap
resampling over instances.
Holm--Bonferroni correction is applied to control family-wise error rate of the above 2 marginal pairwise tests on fully pooled data
at $\alpha=5\%$. %what family wise? Across difficulties? Add tau analysis later.

\paragraph{Difficulty scaling.}
To analyze how performance varies with task difficulty, we fit a generalized linear mixed-effects model (GLMM)
for binary accuracy $Y_i \in \{0,1\}$ with a logistic link:
\[
Y_i \mid u_{\text{inst}[i]}, u_{\text{seed}[i]} \sim \mathrm{Bernoulli}(p_i),
\]
\[
\mathrm{logit}(p_i)
=
\alpha
+
\beta_{\mathrm{route}_i}
+
\gamma\,\tau_i
+
\delta_{\mathrm{route}_i}\,\tau_i
+
u_{\text{inst}[i]}
+
u_{\text{seed}[i]},
\]
where $u_{\text{inst}[i]} \sim \mathcal{N}(0,\sigma^2_{\text{inst}})$ and
$u_{\text{seed}[i]} \sim \mathcal{N}(0,\sigma^2_{\text{seed}})$ are random intercepts.
Route and difficulty are modeled as fixed effects (including a route $\times$ difficulty interaction),
with instance and seed modeled as random effects.

% \begin{figure*}[t]
%     \centering
%     \vspace{-5pt}
%     \reviewfigurepair{images/CleanShot 2026-01-27 at 11.05.15@2x.png}{images/review_20260528/fig5_translator_judge.pdf}
%     \vspace{-3pt}
%     \caption{\textbf{Translation and discrimination prompts.}
%     \textbf{(a) Translator}: Converts code to NL reasoning step-by-step, mimicking native reasoning.
%     \textbf{(b) Discriminator}: Judge models classify traces as ``Native NL'' or ``Translated''.
%     Used in Section~\ref{sec:rep_analysis}.}
%     \label{fig:translator_discriminator}
% \end{figure*}

\paragraph{Results.}
Across the 40-task analysis set (\cref{fig:main}), Route~1 reaches
17.21\% accuracy, Route~2 reaches 17.37\%, and Route~3 reaches 48.84\%.
\emph{Route~3 has a statistically significant advantage over Route~2},
with a +31.47pp paired accuracy gap and a 95\% cluster-bootstrap interval
of $[+29.20,+33.71]$pp. For Route~1 vs.\ Route~2, the paired gap is
+0.15pp with a 95\% cluster-bootstrap interval of $[-0.30,+0.61]$pp
and a two-sided McNemar $p=0.39$, so we do not conclude a meaningful
difference between natural-language reasoning and code simulation under
this evaluation. On the other hand, as performance gaps widen with increasing task difficulty
(\cref{fig:per_task_accuracy}), Route~3 remains substantially higher as
the language-based routes degrade. This motivates the hypothesis that execution is the bottleneck.

% \subsection{Representation-Level Analysis: Does Language Add Information?}
% \label{sec:rep_analysis}

% The theoretical results in Section~\ref{sec:theory} rely on the premise
% that, for the evaluated tasks, models, and prompts, natural-language
% reasoning traces do not contain decision-relevant information beyond
% what is already present in code representations.
% We empirically evaluate this premise by analyzing both distributional
% and functional similarity between native NL reasoning and NL obtained
% by translating code.

\section{Route~1 vs.\ Route~2 Analysis}
\label{sec:rep_analysis}
In \cref{sec:performance}, Route~1 and Route~2 achieved similar paired accuracies under our evaluated models and prompts. This section asks: \textbf{Can we provide evidence that input representation is not the bottleneck to end-task performance?} We first model the theory in the simple linear case (\cref{sec:rep-noninferiority,sec:covariance-ordering}), then test the corresponding language-interface prediction (\cref{sec:func_similarity}).

The main result is that when we model code distribution as a canonicalization of natural language distribution, we can show for our simple linear hypothesis class that the uniform worst-case risk over \emph{all} hypotheses is exactly 0. In other words, under these conditions, ``code'' reasoning is non-inferior to ``natural language'' reasoning.

\subsection{Linear Modeling Setup and Intuition}
\label{sec:rep-noninferiority}

% We study code and native natural-language representations as two input distributions for the same downstream hypothesis class. The theory keeps the executor fixed and places the route difference in the representation distribution.

% We identify a sufficient criteria under which\[
% R^*_{\mathcal{H}_{\mathrm{C}}}(E_{\mathrm{Code}})
% \;\le\;
% R^*_{\mathcal{H}_{\mathrm{NL}}}(E_{\mathrm{NL}})
% +
% \varepsilon.
% \] holds.

Controlling the downstream hypothesis class enables us to compare the effect of the two input representations  \(P_C\) and \(P_N\), both of which live in the same feature space. For convenience, the downstream hypothesis class is modeled as realizable and linear.
\begin{equation}
\cH_{\rm lin}=\big\{h_{a,v}(b,u)=a^\top b+v^\top u:
a\in\R^{d_B},\ v\in\R^{d_U}\big\}.
\label{eq:linear-class-section3}
\end{equation}

Let $r\in\{C,N\}$, where $C$ denotes code and $N$ denotes native natural language. For each task, route $r$ produces a random input:
\begin{equation}
    S_r=(B,U_r)\sim P_r,
    \qquad S_r\in\R^{d_B+d_U}.
\end{equation}
The vector \(B\in\R^{d_B}\) is the answer-relevant core, or the
sufficient statistic \citep{lehmann1998theory}. The vector \(U_r\in\R^{d_U}\) contains route-specific
surface variation that should not change the answer once the core is fixed.
Prior research in NLP / LLMs has shown that models are sensitive to syntactic
heuristics, prompting, lexical choice, etc
\citep{gururangan2018annotation,mccoy2019right,geirhos2020shortcut,zhao2021calibrate,lu2021fantastically,pezeshkpour2023order}.

As a running example, consider a dynamic-programming problem. The core \(B\)
contains the recurrence, boundary conditions, and final query that determine the
answer. A code trace may still vary in decision-irrelevant details such as
variable names, helper-function names, or formatting; these are part of \(U_C\).
A natural-language trace can express the same recurrence, but it also admits
extra paraphrases such as ``fill the table from smaller subproblems,'' ``reuse
previous states,'' or ``work backward from the target.'' These additional
surface choices are the extra term \(\eta_N\). If both routes preserve \(B\),
then those prose choices should not create new algorithmic information.

We define the linear-model assumption here before proving the risk comparison.
Following the intuition presented above, we state the assumption as a shared-core generative model with extra natural-language nuisance. The first-order orthogonality conditions imply the centering and risk-separation facts defined in the lemma.

\begin{assumption}
\label{ass:extra-nuisance}
There exist variables \(B\), \(U_C\), \(\eta_N\), and \(\varepsilon\) such that
\begin{equation}
    S_C=(B,U_C),
    \qquad
    S_N=(B,U_N),
    \qquad
    U_N=U_C+\eta_N,
    \label{eq:extra-nuisance}
\end{equation}
and
\begin{equation}
    Y=\theta^\top B+\varepsilon.
    \label{eq:linear-label-section3}
\end{equation}
The code-side nuisance is centered conditional on the core, the extra
natural-language nuisance is mean-zero conditional on the core and code-side
nuisance, and the residual label noise is mean-zero conditional on all nuisance
coordinates:
\begin{equation}
\begin{aligned}
    \E[U_C\mid B]&=0,\\
    \E[\eta_N\mid B,U_C]&=0,\\
    \E[\varepsilon\mid B,U_C,\eta_N]&=0.
\end{aligned}
\label{eq:first-order-orthogonality}
\end{equation}
The residual nuisance \(\eta_N\) represents additional paraphrase, verbosity,
ordering, style, and prompt-form variation in native natural language. The
direction \(U_N=U_C+\eta_N\) formalizes the claim that natural language carries
the code-side nuisance plus additional surface variation. Code has a more constrained grammar and conventional structure, whereas
natural-language explanations admit more paraphrastic and discourse-level
variation. Work on the naturalness of software likewise treats code as a
structured statistical object with regularities distinct from ordinary natural language
\citep{allamanis2018survey}.
\end{assumption}

The next lemma is necessary and load-bearing for the main theorem. The lemma states that, once the core is fixed, both routes' surface coordinates remain centered nuisance and do not secretly carry residual label information.

\begin{lemma}
\label{lem:shared-core-consequences}
Under \cref{ass:extra-nuisance}, for \(r\in\{C,N\}\),
\begin{equation}
    \E[U_r\mid B]=0,
    \label{eq:centered-section3}
\end{equation}
\begin{equation}
    \E[\varepsilon\mid B,U_r]=0,
    \label{eq:label-noise-section3}
\end{equation}
and
\begin{equation}
    \E[BU_r^\top]=0,
    \qquad
    \E[\varepsilon U_r]=0.
    \label{eq:risk-separation-orthogonality}
\end{equation}
Moreover, the first-order nuisance condition implies the cross-moment identity:
\begin{equation}
    \E[U_C\eta_N^\top\mid B]=0.
    \label{eq:extra-orthogonal}
\end{equation}
\end{lemma}

\begin{proof}
For code, \(\E[U_C\mid B]=0\) is part of \cref{ass:extra-nuisance}. For natural
language,
\[
\begin{aligned}
    \E[U_N\mid B]
    &= \E[U_C\mid B]+\E[\eta_N\mid B]\\
    &= 0+\E[\E[\eta_N\mid B,U_C]\mid B]\\
    &= 0.
\end{aligned}
\]
The label-noise condition follows by the tower property because \(U_C\) and
\(U_N=U_C+\eta_N\) are both functions of \((B,U_C,\eta_N)\). Thus
\(\E[\varepsilon\mid B,U_r]=0\) for \(r\in\{C,N\}\). The two risk-separation
equalities follow from
\[
    \E[BU_r^\top]=\E[B\,\E[U_r^\top\mid B]]
\]
and
\[
    \E[\varepsilon U_r]
    =
    \E[U_r\,\E[\varepsilon\mid B,U_r]]
    =
    0.
\]
Finally,
\[
    \E[U_C\eta_N^\top\mid B]
    =
    \E[U_C\,\E[\eta_N^\top\mid B,U_C]\mid B]
    =
    0.
\]
\end{proof}

Intuitively, the lemma says that the extra natural-language words are not
allowed to be a hidden answer channel once the algorithmic core is known.

\subsection{Results}
We now make the representation claim precise in the linear model. The argument
has three steps: define a nuisance covariance ordering, show that the
extra-native-language-nuisance assumption implies that ordering, and then use
the ordering to compare the risk of every member of the linear hypothesis class
under the two route distributions. We end with intuition that leads to the experimental results.

\label{sec:covariance-ordering}

Let
\begin{equation}
    \Sigma_{U,r}=\E[U_rU_r^\top]
    \label{eq:nuisance-covariance}
\end{equation}
be the route-$r$ nuisance covariance. We use the Loewner order on positive
semidefinite matrices: $A\preceq B$ means $B-A$ is positive semidefinite,
equivalently $v^\top A v\le v^\top Bv$ for every vector $v$.

The following proposition formalizes the following intuition. If native natural language is the code-side nuisance plus extra paraphrase, then the
natural-language nuisance covariance cloud should be at least as spread out as the code nuisance cloud in every linear direction. In other words, conditional on the same core, the additional natural-language degrees of freedom do not supply extra answer signal.

\begin{proposition}
\label{prop:extra-implies-cov}
Under \cref{ass:extra-nuisance},
\begin{equation}
    \Sigma_{U,C}\preceq \Sigma_{U,N}.
    \label{eq:cov-order}
\end{equation}
\end{proposition}

\begin{proof}
Using \eqref{eq:extra-nuisance},
\begin{align}
\Sigma_{U,N}
&=\E[(U_C+\eta_N)(U_C+\eta_N)^\top]\\
&=\E[U_CU_C^\top]+\E[\eta_N\eta_N^\top]
  +\E[U_C\eta_N^\top]+\E[\eta_N U_C^\top].
\end{align}
The cross terms vanish by \cref{lem:shared-core-consequences}:
\begin{equation}
\E[U_C\eta_N^\top]
=\E\big[\E[U_C\eta_N^\top\mid B]\big]=0,
\end{equation}
and similarly for its transpose. Hence
\begin{equation}
    \Sigma_{U,N}-\Sigma_{U,C}=\E[\eta_N\eta_N^\top]\succeq 0.
\end{equation}
\end{proof}

The covariance order from the above proposition implies the uniform population risk non-inferiority for code representations.

\begin{theorem}
\label{thm:cov-ordering}
Under \cref{ass:extra-nuisance}, for squared loss and the common class $\cH_{\rm lin}$ in \eqref{eq:linear-class-section3},
\begin{equation}
    \sup_{h\in\cH_{\rm lin}}\{R_C(h)-R_N(h)\}\le 0.
    \label{eq:uniform-ordering}
\end{equation}
Consequently code is representation-non-inferior to native natural language with margin 0 in this linear model.
\end{theorem}

\begin{proof}
Fix $h_{a,v}\in\cH_{\rm lin}$. By \eqref{eq:linear-label-section3},
\begin{equation}
    h_{a,v}(B,U_r)-Y=(a-\theta)^\top B+v^\top U_r-\varepsilon .
\end{equation}
Using \cref{lem:shared-core-consequences}, the terms involving $U_r$ separate:
\begin{equation}
    R_r(h_{a,v})=R_{\rm core}(a)+v^\top\Sigma_{U,r}v,
\end{equation}
where $R_{\rm core}(a)=\E[((a-\theta)^\top B-\varepsilon)^2]$ does not depend on $r$. Therefore
\begin{equation}
    R_C(h_{a,v})-R_N(h_{a,v})
    =v^\top(\Sigma_{U,C}-\Sigma_{U,N})v\le 0
\end{equation}
by \eqref{eq:cov-order}. Taking the supremum over $h\in\cH_{\rm lin}$ proves \eqref{eq:uniform-ordering}.
\end{proof}

Since the core term is identical for code and native
natural language, the only possible representation-level difference is the
penalty for depending on \(U_r\). Because native language has no smaller
nuisance covariance, it cannot lower that penalty. Thus, in this model,
trace representation in Route~1 cannot be the bottleneck for improving performance, since the core algorithmic information is already present in the code representation \footnote{If this were not the case, we would not be able to get the final answer output from our pipeline.}.

\begin{remark}[Finite-sample OLS intuition]
By studying why nuisance coordinates hurt learning in the finite-sample regime, we can gain intuition about the uniform population risk result above\footnote{In the uniform case, we do not select a hypothesis in the hypothesis class like in the learning case, but rather, we show the phenomenon is true across \emph{all} hypotheses in the class.}. Let $p=d_B+d_U$ and consider the Gaussian special case
\begin{equation}
    S_r\sim\mathcal N(0,\Sigma_r),\qquad
    Y=\beta^{\star\top}S_r+\varepsilon,
    \qquad \beta^\star=(\theta,0),
\end{equation}
with $\varepsilon\sim\mathcal N(0,\sigma^2)$. Ordinary least squares and ridge regression have finite-sample excess-risk terms controlled by dimension, conditioning, and noise variance. In the isotropic case $\Sigma_r=I_p$, if $\widehat\beta_r$ is ordinary least squares from $n>p+1$ independent route-$r$ samples, then
\begin{equation}
    \E_{D_{r,n}}[R_r(\widehat\beta_r)]-R^\star
    =\sigma^2\frac{p}{n-p-1},
    \label{eq:ols-excess}
\end{equation}
where $R^\star=\sigma^2$. Fitting irrelevant coordinates increases finite-sample prediction error \citep{hastie2009elements,wainwright2019high}.
In the running example, this is the difference between learning from the
recurrence itself and learning from many alternative descriptions of the same
recurrence. Extra wording can be harmless in the infinite-data idealization. However, it is not a new source of answer-relevant information and can be a nuisance channel that a finite model must learn not to use.
\end{remark}

Overall, the linear theory provides an intuition for why representation is not the bottleneck. Once both routes preserve the same core, extra native-language nuisance can only increase population risk for this hypothesis class. \Cref{sec:func_similarity} demonstrates the nuisance intuition empirically when the theory breaks down.

\begin{figure}[t]
    \centering
    \vspace{-5pt}
    \includegraphics[width=\linewidth]{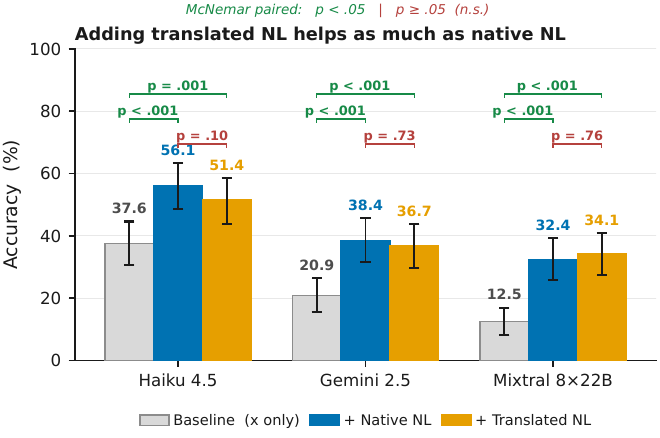}
    \vspace{-28pt}
    \caption{\textbf{Translated NL has similar downstream functionality to native NL as measured by end-task accuracy.} When we concatenate reasoning traces to the prompt, accuracy rises above the question-only baseline, indicating that the model uses the reasoning. Baseline-vs-concatenated differences are statistically significant, while Translated NL and Native NL are not significantly different. Source data is randomly sampled from the data used to generate \cref{fig:main}; we use 1000 samples per model. }
    \label{fig:translation_additivity}
\end{figure}

\subsection{Decision Intervention Reduction}
\label{sec:func_similarity}
When the theory breaks down due to non-linearity and
high-dimensionality, we show that the practical implications remain. If translating code into natural language via fixed transformation, similar to our nuisance intuition, preserves downstream accuracy relative to native
natural-language reasoning, then the representation non-inferiority of code remains the same even outside the linear proof.

\paragraph{Intervention setup.}
Our goal is to provide evidence that code is non-inferior to NL in expected loss in the LLM regime. In our proof, we modeled NL as a fixed transformation (linearly additive projection) of a code latent. Similarly, we apply a fixed transformation of the code (e.g., permutations and syntactic shifts) to simulate NL. We expect NL to yield nothing worse, since the underlying assumption is that code contains all decision-relevant information. The following experiment tests this implication directly.

We show that a fixed transformation of code (to natural language) behaves
functionally similar to original natural-language reasoning. Here, our estimand
is final accuracy, and the treatment is one of three inputs to the LLM executor:
translated natural-language reasoning, original natural-language reasoning, or
the baseline prompt without reasoning.

For a task instance $x$, we prompt a target language model in three conditions:
(1) Baseline:  x,
(2) Native: $ x \,\|\, z_{\mathrm{NL}},$
(3) Translated:  x $\,\|\, \hat z_{\mathrm{NL}},$ where $z_{\mathrm{NL}} \sim p_{\mathrm{NL}}(\cdot\mid x)$ is the native Route~1 trace and
$\hat z_{\mathrm{NL}}$ is obtained by translating the corresponding code to NL using the same translator procedure.
If translated NL loses decision-relevant information relative to native NL, we would expect
\[
\mathrm{Acc}(x \,\|\, \hat z_{\mathrm{NL}}) < \mathrm{Acc}(x \,\|\, z_{\mathrm{NL}}).
\]

\paragraph{Protocol and models.}
We run this test on 1000 held-out instances across multiple translator models (\texttt{Claude-Haiku-4.5, Gemini-2.5-Flash, Mixtral-8x22B-Instruct}).
Crucially, in this experiment the translator model is matched to the original code generator (i.e., we translate code produced by the same model family), to avoid confounding functional losses with cross-model stylistic mismatch.

\paragraph{Results.}
We fail to reject the null hypothesis of equal performance and observe overlapping 95\% confidence intervals between the Native and Translated conditions (\cref{fig:translation_additivity}).
This suggests that, under our evaluated settings, translating code into NL does not destroy decision-relevant information for downstream answering, consistent with the claim that NL CoT does not systematically add algorithmic advantage beyond what is in code. Qualitatively, the translated results are similar to the originals, though this appears to depend heavily on the translating prompt.

\paragraph{Interpretation.} Our empirical results indicate that translated code traces and native natural-language traces provide similar decision-relevant information in the settings we study. We find no evidence that natural-language reasoning introduces novel algorithmic strategies beyond those already captured by code representations on algorithmic tasks. Thus, in these experiments, replacing NL traces with code traces in the generation stage is not the end-to-end bottleneck. The execution analysis in \cref{sec:route2v3} examines the remaining gap.

\section{Route~2 vs.\ Route~3 Analysis}
\label{sec:route2v3}
The key research question in this section is: \textbf{Is execution the bottleneck for language-based reasoning?}
To address this question, we also ask two sub-questions:
\begin{enumerate}
    \item[1)] When code is \emph{correct}, does external code execution (Route 3) have lower expected loss than LLM-based code execution (Route 2)?
    \item[2)] When code is \emph{incorrect}, is the probability that Route 2 outperforms Route 3 small?
\end{enumerate}
We pinpoint execution as the bottleneck (cf.\ \cref{sec:rep_analysis}), highlighting how using LLMs to generate code plans and then delegating to an external solver is the best-performing of the three evaluated routes (\cref{sec:performance}). In the 40-task analysis set, Route~3 exceeds Route~2 by +31.47pp. The recovery mass where Route~2 succeeds while Route~3 fails is 1.61\%, whereas the execution-win mass where Route~3 succeeds while Route~2 fails is 33.08\%.

\subsection{Setup}
\label{sec:route2v3_setup}
As defined in \cref{sec:framework}, Route~2 and Route~3 share the same
trace generator $E_{\mathrm{Code}}$ and differ only in the executor.

Let $\rho_{\mathrm{Sim}}\in\mathcal{H}_{\mathrm{C}}$ denote the fixed
LLM-based simulation executor used in Route~2.
Let $g$ be a deterministic interpreter mapping $(x,z_{\mathrm{C}})$ to
$\mathcal{Y}\cup\{\bot\}$, where $\bot$ denotes execution failure. Define the instance-correct execution event
\[
C := \{ g(X,Z_{\mathrm{C}}) = Y^*(X) \}.
\]

\paragraph{Risk decomposition.}
Let
\begin{align*}
& e_C := \Pr(\hat Y_{\mathrm{Sim}}\neq Y^*(X)\mid C),
\qquad \\
& r := \Pr(\hat Y_{\mathrm{Sim}}=Y^*(X)\mid \neg C).
\end{align*}
Then
\begin{align*}
R(E_{\mathrm{Code}},\rho_{\mathrm{Exec}}) &= \Pr(\neg C),\\
R(E_{\mathrm{Code}},\rho_{\mathrm{Sim}}) &= \Pr(C)\,e_C + \Pr(\neg C)(1-r),
\end{align*}
and hence
\begin{align*}
&R(E_{\mathrm{Code}},\rho_{\mathrm{Sim}})
 -R(E_{\mathrm{Code}},\rho_{\mathrm{Exec}})\\
&\qquad =
\Pr(C)\,e_C
-
\Pr(\neg C)\,r .
\end{align*}

\paragraph{Implications.}
Simulation noise on correct-execution instances ($\Pr(C)e_C$) harms
Route~2, while recovery on incorrect or failed executions
($\Pr(\neg C)r$) helps Route~2.
Route~3 dominates whenever the recovery mass is insufficient to offset
simulation noise, a condition quantified empirically in
\cref{sec:performance}.

\paragraph{Empirical recovery as an upper bound.}
A direct way to operationalize the ``recovery'' term in the decomposition is to measure the frequency with which Route~2 succeeds while Route~3 fails on the \emph{same} generated code trace, i.e.,
\[
\Pr\!\left(
\hat Y_{\mathrm{Sim}} = Y^*(X),
\;
g(X,Z_{\mathrm{C}})\neq Y^*(X)
\right).
\]
This quantity corresponds to the \emph{recovery mass} $\Pr(\neg C)\,r$ appearing in the above risk gap.
% \[
% R(E_{\mathrm{Code}},\rho_{\mathrm{Sim}})
% -
% R(E_{\mathrm{Code}},\rho_{\mathrm{Exec}})
% =
% \Pr(C)\,e_C
% -
% \Pr(\neg C)\,r.
% \]
It aggregates both (i) genuine recovery from incorrect code and (ii) cases where execution fails (e.g., $g(X,Z_{\mathrm{C}})=\bot$) but the simulator still answers correctly.
Because the simulator may partially ignore $Z_{\mathrm{C}}$ and answer directly from $X$, this statistic should be interpreted as an \emph{upper bound} on mechanistic ``recovery from flawed code.''

\paragraph{Interpretation.}
Route~2 can outperform Route~3 only if the recovery mass $\Pr(\neg C)\,r$ is large enough to offset simulation noise on correct-execution instances $\Pr(C)\,e_C$.
Empirically, we find that recovery mass is consistently small across tasks and models (\cref{fig:recovery_final}), indicating that Route~2 rarely compensates for execution failures\footnote{\Cref{tab:appendix_code_failure_modes} reports the corresponding code-execution failure modes.} via recovery.
In the 40-task analysis set, this recovery mass is 1.61\%, and the execution-win mass $\Pr(C)e_C$ is 33.08\%. These values explain why deterministic execution typically achieves lower end-to-end error than simulation on the same generated code traces (\cref{fig:main}).

\begin{figure}[t]
    \centering
    \vspace{-5pt}
    \includegraphics[width=\linewidth]{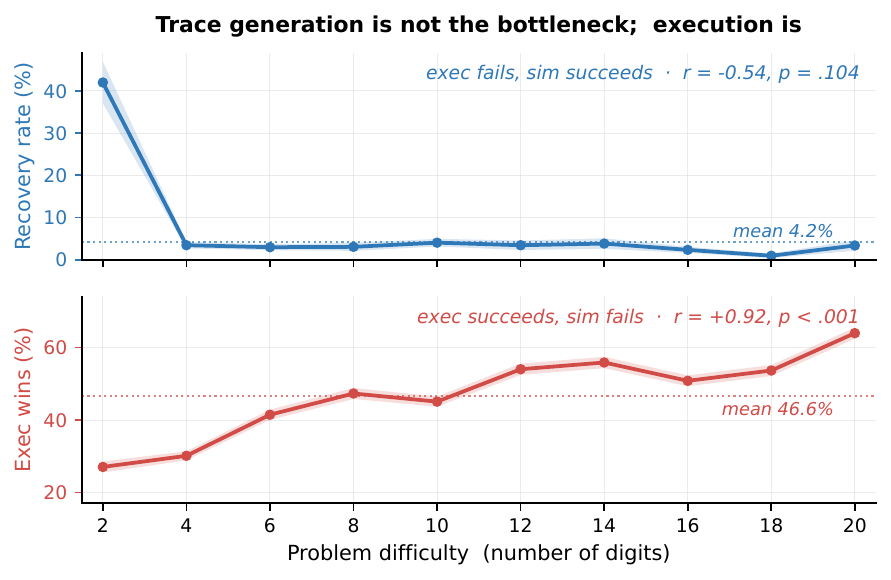}
    \vspace{-28pt}
    \caption{\textbf{Recovery mass (blue) is 1.61\% overall, while execution-win mass (red) is 33.08\% overall.} Recovery mass $\Pr(\neg C)\,r$ counts cases where code execution fails but LLM simulation succeeds. Execution-win mass $\Pr(C)e_C$ counts cases where code execution succeeds and LLM simulation fails.}
    \label{fig:recovery_final}
\end{figure}

% \subsection{Relating Theory and Empirical Observations}

\section{Related Work and Discussion}
\label{sec:related_work}

\textbf{Neuro-symbolic Learning.} This paper builds on research in neuro-symbolic integration \cite{graves_neural_2014, velickovic_neural_2021, reed_neural_2016, graves_hybrid_2016}, which combines neural networks with symbolic reasoning systems. These approaches are motivated by cognitive science \cite{schneider_controlled_2003, risko_cognitive_2016, anderson_neural_2010}, hierarchical reinforcement learning \cite{kolter_hierarchical_2007, dietterich_hierarchical_2000}, and compositionality research \cite{hudson_compositional_2018, hupkes_compositionality_2020, andreas_neural_2017, poggio2017and}. An orthogonal line of work explores direct execution of algorithms by neural networks \cite{velickovic_neural_2021, mahdavi_towards_2023, ibarz_generalist_2022, yan_neural_2020}. Unlike these approaches that focus on \emph{how} to integrate neural and symbolic components, our work addresses \emph{why} neuro-symbolic integration outperforms neural reasoning alone for algorithmic tasks (\cref{sec:rep_analysis,sec:route2v3}).

\textbf{LLM Reasoning.} Recent work has explored various reasoning paradigms for LLMs, including symbolic reasoning \cite{marra_integrating_2019, olausson_linc_2023, han_folio_2024}, chain-of-thought prompting \cite{wei2022chain, zelikman_star_2022, merrill_expressive_2024, altabaa_cot_2025}, and in-context learning \cite{xie2021explanation, garg2022can, akyurek2022learning, zhang2024trained}. \citet{xie2021explanation} model in-context learning as implicit Bayesian inference, which we extend to compare different reasoning representations. While prior work demonstrates \emph{that} certain prompting strategies improve performance, we provide a theoretical framework (\cref{sec:framework}) explaining \emph{why} code representations are never worse than natural language (\cref{sec:rep_analysis}) in certain settings.

\textbf{LLM Tool-Use.} Tool-augmented LLMs have achieved strong empirical results \cite{shen_llm_2024, 10.5555/3666122.3669119, qin_toolllm_2023, tang_toolalpaca_2023, parisi_talm_2022}. Code generation for tool-use can be viewed as a form of semantic parsing \cite{shin_few-shot_2022, krishnamurthy_neural_2017, berant_semantic_2013, dong_language_2016} or function calling \cite{puri_codenet_2021, alon_code2vec_2019, chen_neural_2018}. Our work complements this literature by providing theoretical justification (\cref{sec:rep_analysis,sec:route2v3}) for the observed empirical advantages of code-based tool-use over direct natural language reasoning. 

% \section{Discussion}

% \textbf{Summary.} This paper addresses whether algorithmic problems encoded in natural language should be solved via direct reasoning or by translating to code and executing with a solver. Our three-arm framework demonstrates that code execution consistently outperforms both code simulation and natural language reasoning, with theoretical backing from information theory.

% \textbf{Interpretation.} The theoretical analysis reveals that code representations yield higher mutual information with target algorithms than natural language representations. This explains the empirical observation: code serves as a more discriminative intermediate representation for implicit algorithm classification during LLM inference. The Bayesian framework makes this comparison tractable by decomposing the problem into translation and execution phases.

\section{Conclusion}
\label{sec:conclusion}
We introduced a three-route framework (\cref{sec:framework}) for disentangling representation and execution in algorithmic reasoning. The intermediate route, code generation followed by LLM simulation, lets us compare natural-language reasoning and code execution without changing both the trace representation and the executor at once. On the 40 algorithmic benchmark tasks, Route~1 and Route~2 are statistically close: Route~2 exceeds Route~1 by +0.15pp with a 95\% cluster-bootstrap CI of $[-0.30,+0.61]$pp. Route~3 reaches 48.84\% accuracy and exceeds Route~2 by +31.47pp with CI $[+29.20,+33.71]$pp (\cref{sec:performance}). The representation analysis in \cref{sec:rep_analysis} explains the small Route~1/Route~2 gap through a shared-core nuisance model and a reconstruction intervention showing that code-derived natural-language traces retain comparable downstream utility. The execution analysis in \cref{sec:route2v3} explains the large Route~3 gain through low recovery mass (1.61\%) and high execution-win mass (33.08\%, \cref{fig:recovery_final}). Taken together, these results suggest that code helps primarily by generating an executable trace that can be run reliably. In this setting, the central advantage of tool use is not the code representation, but the ability to hand the generated trace to a deterministic executor.

% Our causal intervention experiments (\cref{sec:rep_analysis}) demonstrate that NL reasoning traces are distributionally and functionally equivalent to code-translated traces, supporting the hypothesis that NL reasoning implicitly simulates underlying algorithmic computations for specific models, tasks and prompts. Theoretically, we prove that code-based reasoning achieves lower Bayes risk via Blackwell dominance (\cref{prop:dominance}), providing an information-theoretic explanation for the empirical advantages of solver-based pipelines.

\paragraph{Limitations.}
We highlight a few limitations. First, the main experiments do not fully cover frontier reasoning models. We ran controlled-subset frontier-model ablations (\Cref{tab:appendix_frontier_nopatch}) and found that the results generally hold. Second, task coverage is limited to algorithmic problems with externally verifiable answers and reliable Python execution, so open-ended, ambiguous, unsafe, or weakly specified tool use may behave differently. Third, the theory mainly builds intuition through standard decision-theoretic ideas and derives its novelty from the application, rather than providing a new general theorem about LLMs or code. Fourth, the linear models and shared-core/sufficient-statistic assumptions are motivated by practice but may fail when code omits decision-relevant information, natural language carries useful signal, or executors exploit structure outside the modeled core.

% \paragraph{Limitations}

% \paragraph{Future Work}

\section*{Impact Statement}
This paper presents work whose goal is to advance the field of machine learning. There are many potential societal consequences of our work, none of which we feel must be specifically highlighted here.

\section*{Acknowledgements}
This work was partially funded by ONR Contract N00014-23-1-2417. We are grateful for resources and computational support provided by the Cognitive Computation Group at the University of Pennsylvania. We also thank the reviewers for their thoughtful feedback and comments.

\bibliography{example_paper}
\bibliographystyle{icml2025}

%%%%%%%%%%%%%%%%%%%%%%%%%%%%%%%%%%%%%%%%%%%%%%%%%%%%%%%%%%%%%%%%%%%%%%%%%%%%%%%
%%%%%%%%%%%%%%%%%%%%%%%%%%%%%%%%%%%%%%%%%%%%%%%%%%%%%%%%%%%%%%%%%%%%%%%%%%%%%%%
% APPENDIX
%%%%%%%%%%%%%%%%%%%%%%%%%%%%%%%%%%%%%%%%%%%%%%%%%%%%%%%%%%%%%%%%%%%%%%%%%%%%%%%
%%%%%%%%%%%%%%%%%%%%%%%%%%%%%%%%%%%%%%%%%%%%%%%%%%%%%%%%%%%%%%%%%%%%%%%%%%%%%%%
\newpage
\appendix
% Polished LaTeX appendix source included by example_paper.tex.
% Required packages in the main preamble: booktabs, array, caption, url, hyperref, cleveref.

\clearpage
\onecolumn

\captionsetup[table]{font=small,labelfont=bf,textfont=normalfont,labelsep=colon,skip=5pt,hypcap=false}

% Appendix-local notation and formatting helpers.
\newcommand{\RouteOne}{Route~1}
\newcommand{\RouteTwo}{Route~2}
\newcommand{\RouteThree}{Route~3}
\newcommand{\appendixresult}[1]{\par\noindent\textbf{Result.}~#1\par}
\newenvironment{appendixtableblock}{\par\smallskip\noindent\begin{minipage}{\linewidth}\centering}{\end{minipage}\par\smallskip}

\section{Supplementary Route Results}
\label{app:route_robustness}

\subsection{Complexity-Stratified Route Accuracy}
\label{app:complexity_stratified_route_accuracy}

We stratify the main route-accuracy evaluation by the asymptotic complexity
class of each task to check whether a single complexity regime drives the
aggregate route pattern.

\begin{appendixtableblock}
\footnotesize
\setlength{\tabcolsep}{4pt}
\begin{tabular*}{\linewidth}{@{\extracolsep{\fill}}lrrrrrrr@{}}
\toprule
Complexity & Tasks & Inst. & R1 & R2 & R3 & R2--R1 & R3--R2 \\
\midrule
$O(1)$ & 3 & 180 & 46.2 & 47.3 & 86.2 & 1.0 & 39.0 \\
$O(\log n)$ & 1 & 35 & 22.4 & 28.6 & 32.9 & 6.2 & 4.3 \\
$O(n)$ & 4 & 135 & 4.8 & 5.0 & 7.5 & 0.2 & 2.5 \\
$O(n \log n)$ & 5 & 68 & 0.0 & 0.0 & 0.0 & 0.0 & 0.0 \\
$O(n^2)$ & 16 & 297 & 10.3 & 10.5 & 38.9 & 0.1 & 28.5 \\
$O(n^2 \log n)$ & 1 & 1 & 0.0 & 0.0 & 0.0 & 0.0 & 0.0 \\
$O(n^3)$ & 4 & 37 & 0.0 & 0.0 & 0.0 & 0.0 & 0.0 \\
NP-hard & 6 & 360 & 17.6 & 16.8 & 69.7 & -0.8 & 53.0 \\
\bottomrule
\end{tabular*}
\captionof{table}{Route accuracy grouped by asymptotic complexity on the 40-task benchmark set in the main route evaluation. We report task and instance counts. The route differences are percentage-point gaps.}
\label{tab:appendix_complexity_routes}
\end{appendixtableblock}

\appendixresult{\Cref{tab:appendix_complexity_routes} shows the aggregate route pattern within the larger retained complexity strata. \RouteThree{}--\RouteTwo{} is the largest observed route change, and \RouteTwo{} stays close to \RouteOne{}. }

\noindent\textbf{Task mapping.}~The overall task set's largest strata are $O(n^2)$ dynamic-programming, sorting, string, graph, and
shortest-path tasks, plus NP-hard ILP, edge-disjoint paths, shortest-path, and
TSP tasks. Smaller strata cover constant-time arithmetic, logarithmic binary
search, linear selection/string matching,
$O(n \log n)$ scheduling/sorting/hulls, $O(n^2 \log n)$ Kruskal MST, and cubic
dynamic-programming/shortest-path routines.

\subsection{Model-Stratified Route Accuracy}
\label{app:model_stratified_route_accuracy}

We stratify the main route comparison by model to check whether the aggregate
result is driven by one model family or scale regime. The \RouteThree{} advantage
persists across all six full-coverage models, while the \RouteTwo{}--\RouteOne{}
difference is small and changes sign across models.

\begin{appendixtableblock}
\footnotesize
\setlength{\tabcolsep}{2pt}
\begin{tabular*}{\linewidth}{@{\extracolsep{\fill}}>{\raggedright\arraybackslash}p{0.36\linewidth}lrrrrrrrr@{}}
\toprule
Model & Type & Inst. & R1 & R2 & R3 & R2--R1 & $p_{2,1}$ & R3--R2 & $p_{3,2}$ \\
\midrule
\path{anthropic/claude-haiku-4.5} & closed & 1113 & 23.3 & 23.6 & 47.1 & 0.3 & 0.612 & 23.5 & $2.4{\times}10^{-148}$ \\
\path{google/gemini-2.0-flash-001} & closed & 1113 & 17.5 & 18.1 & 51.8 & 0.7 & 0.137 & 33.7 & $8.5{\times}10^{-298}$ \\
\path{google/gemini-2.5-flash} & closed & 1113 & 20.7 & 20.5 & 50.2 & -0.1 & 0.823 & 29.7 & $3.1{\times}10^{-238}$ \\
\path{openai/gpt-4o-mini} & closed & 1113 & 14.1 & 12.9 & 49.1 & -1.2 & 0.00663 & 36.1 & $<10^{-300}$ \\
\path{mistralai/codestral-2508} & open & 1113 & 15.0 & 15.8 & 52.7 & 0.8 & 0.0279 & 36.9 & $<10^{-300}$ \\
\path{mistralai/mixtral-8x22b-instruct} & open & 1113 & 12.7 & 13.2 & 42.1 & 0.5 & 0.127 & 28.9 & $7.2{\times}10^{-253}$ \\
\bottomrule
\end{tabular*}
\captionof{table}{Route accuracy grouped by model on the 40-task analysis set. Each row contains 1,113 unique problems and 3,339 paired route evaluations; $p_{2,1}$ and $p_{3,2}$ are exact two-sided McNemar tests for \RouteTwo{} vs. \RouteOne{} and \RouteThree{} vs. \RouteTwo{}, respectively.}
\label{tab:appendix_model_routes}
\end{appendixtableblock}

\appendixresult{\Cref{tab:appendix_model_routes} shows that \RouteThree{} stays above \RouteTwo{} for every full-coverage model row, so the aggregate execution gain is not driven by a single weak or strong model.}

\section{Functional Similarity Checks}
\label{app:functional_similarity_checks}

\subsection{Prompt Translation Shot Ablation}
\label{app:prompt_translation_shot_ablation}

The functional similarity test translates code traces back into natural
language as a fixed transformation. To test sensitivity to prompting, we vary the number of
in-context examples used by the translator and measure translated-trace utility
against native natural-language traces.

\begin{appendixtableblock}
\small
\setlength{\tabcolsep}{4pt}
\begin{tabular*}{\linewidth}{@{\extracolsep{\fill}}lrrrrrr@{}}
\toprule
Shots & $x$ & $x$ + native NL & $x$ + translated NL & $\Delta$ native & $\Delta$ translated & Gap \\
\midrule
0 & 32.70\% & 55.70\% & 42.41\% & +23.00pp & +9.70pp & -13.29pp \\
1 & 32.70\% & 55.70\% & 45.78\% & +23.00pp & +13.08pp & -9.92pp \\
2 & 32.70\% & 55.70\% & 49.58\% & +23.00pp & +16.88pp & -6.12pp \\
3 & 32.70\% & 55.70\% & 48.52\% & +23.00pp & +15.82pp & -7.17pp \\
4 & 32.70\% & 55.70\% & 49.37\% & +23.00pp & +16.67pp & -6.33pp \\
5 & 32.70\% & 55.70\% & 50.00\% & +23.00pp & +17.30pp & -5.70pp \\
10 (reference) & 39.00\% & 56.50\% & 52.00\% & +17.50pp & +13.00pp & -4.50pp \\
\bottomrule
\end{tabular*}
\captionof{table}{Translation-additivity shot ablation for Claude Haiku 4.5. Here $x$ denotes the question-only baseline, and Gap is translated NL minus native NL. Rows 0--5 use a matched 25\% subset sweep. The 10-shot row is an original main-evaluation reference and is not from the same 25\% subset sweep.}
\label{tab:appendix_shot_ablation}
\end{appendixtableblock}

\appendixresult{\Cref{tab:appendix_shot_ablation} shows that translated traces improve over the question-only baseline at every reported shot setting, but remain below native natural-language traces in every row.}

\section{Additional Model Evaluations}
\label{app:additional_model_evaluations}

\subsection{Recursive Language Model Evaluation}
\label{app:recursive_language_model_evaluation}

To broaden coverage beyond standard LLM forward passes, we evaluate recursive
language model (RLM) execution on a 25\% subset at seed~0. The RLM code and
natural-language arms remain close under this executor family.

\begin{appendixtableblock}
\small
\begin{tabular*}{\linewidth}{@{\extracolsep{\fill}}lrrr@{}}
\toprule
Model & RLM Code Acc & RLM NL Acc & Better Arm \\
\midrule
Claude Haiku 4.5 & 29.4\% & 30.9\% & NL \\
Gemini 2.5 Pro & 47.5\% & 45.2\% & Code \\
\bottomrule
\end{tabular*}
\captionof{table}{Recursive language model code-vs.-NL accuracy on the 25\% subset, seed~0. Better Arm marks the higher row accuracy.}
\label{tab:appendix_rlm}
\end{appendixtableblock}

\appendixresult{\Cref{tab:appendix_rlm} shows mixed code-vs.-NL results under RLM execution: the NL arm is higher for Claude Haiku 4.5, while the code arm is higher for Gemini 2.5 Pro. Because this is a 25\% subset, we use it as model-coverage evidence rather than as the main route estimate.}

\subsection{Coding-Specialized Model Evaluation}
\label{app:coding_specialized_model_evaluation}

We evaluate coding-specialized models to test whether models trained on code
change the \RouteTwo{}--\RouteOne{} pattern. Even for these models, replacing
natural-language reasoning with simulation over code traces changes little
relative to replacing simulation with actual code execution.

\begin{appendixtableblock}
\footnotesize
\setlength{\tabcolsep}{4pt}
\begin{tabular*}{\linewidth}{@{\extracolsep{\fill}}>{\raggedright\arraybackslash}p{0.39\linewidth}ccc@{}}
\toprule
Model & NL & Sim & Code Exec \\
\midrule
\path{x-ai/grok-code-fast-1} (25\% data) & 47.71\% (167/350) & 47.71\% (167/350) & 55.99\% (159/284) \\
\path{qwen/qwen3-coder} (25\% data) & 30.23\% (104/344) & 26.86\% (94/350) & 56.19\% (168/299) \\
\path{codestral-2508} (original) & 19.89\% (943/4740) & 23.14\% (1097/4740) & 59.65\% (2266/3799) \\
\bottomrule
\end{tabular*}
\captionof{table}{Route accuracy for coding-specialized models. NL, Sim, and Code Exec correspond to \RouteOne{}, \RouteTwo{}, and \RouteThree{}; each cell reports accuracy with correct/denominator counts, and denominators are parse-normalized per arm.}
\label{tab:appendix_coding_models}
\end{appendixtableblock}

\appendixresult{\Cref{tab:appendix_coding_models} shows that Code Exec is the highest-accuracy arm in each coding-specialized row. Grok and Qwen use 25\% subset runs, while Codestral uses the original evaluation run.}

\subsection{Frontier Model Controlled-Subset Evaluation}
\label{app:frontier_model_sanity_checks}

\begin{appendixtableblock}
\small
\setlength{\tabcolsep}{5pt}
\begin{tabular*}{\linewidth}{@{\extracolsep{\fill}}lrrr@{}}
\toprule
Model & Route 1 & Route 2 & Route 3 \\
\midrule
GPT-5.4 & 42.57\% (149/350) & 41.43\% (145/350) & 54.01\% (175/324) \\
Claude Opus 4.6 & 52.73\% (174/330) & 73.42\% (116/158) & 77.54\% (107/138) \\
\bottomrule
\end{tabular*}
\captionof{table}{Frontier controlled-subset route accuracy on the seed~1, 350-instance subset.}
\label{tab:appendix_frontier_nopatch}
\end{appendixtableblock}

\appendixresult{\Cref{tab:appendix_frontier_nopatch} keeps \RouteThree{} ahead in both controlled-subset rows, but the margins are smaller than in the main six-model evaluation.}

\clearpage
\section{Failure Analysis}
\label{app:failure_analysis}

\subsection{Code Execution Failure Modes}
\label{app:code_execution_failure_modes}

Among \RouteThree{} code-execution failures, most failures are semantic wrong
answers rather than syntax, runtime, or timeout failures. This supports the interpretation
that these errors often originate in the generated program specification rather
than in the Python runtime itself.

\begin{appendixtableblock}
\small
\begin{tabular*}{\linewidth}{@{\extracolsep{\fill}}lrrrr@{}}
\toprule
Model & Wrong answer & Syntax error & Runtime error & Time limit \\
\midrule
Haiku & 56.26\% & 10.07\% & 1.88\% & 31.80\% \\
Codestral & 81.80\% & 5.87\% & 9.00\% & 3.33\% \\
Gemini 2.0 & 77.22\% & 14.65\% & 1.86\% & 6.27\% \\
Gemini 2.5 & 77.63\% & 17.10\% & 1.82\% & 3.45\% \\
GPT-4o mini & 72.93\% & 4.56\% & 18.97\% & 3.54\% \\
Mixtral & 60.67\% & 4.95\% & 32.14\% & 2.24\% \\
Total & 70.43\% & 9.34\% & 11.55\% & 8.67\% \\
\bottomrule
\end{tabular*}
\captionof{table}{Failure-mode distribution conditional on \RouteThree{} code-execution failures. Rows are within-model percentages.}
\label{tab:appendix_code_failure_modes}
\end{appendixtableblock}

\appendixresult{\Cref{tab:appendix_code_failure_modes} shows that wrong answers dominate among \RouteThree{} code-execution failures. Parse-only rows, where execution completed but the returned value failed parsing or type checks, are excluded to keep the diagnostic focused on failures after execution was attempted.}

% \section{Proof of Corollary 1} \label{app:cor1_proof}

% \begin{proof}
% Let $L_r = \max_\theta \mb E_{\gamma,z} \log q(\gamma| z,\theta, r)$ denote the optimal expected log-likelihood for representation $r$. Note that $H_{CE}(r) = -L_r$.

% From \cref{lem:2}, the mutual information is lower bounded by:
% \begin{align}
%     \mc I(\gamma, Z_r) \geq H(\gamma) + L_r.
% \end{align}

% Taking the difference between code and NL:
% \begin{align}
%     &\mc I(\gamma, Z_{\mathrm{Code}}) - \mc I(\gamma, Z_{\mathrm{NL}}) \nonumber \\
%     &\quad \geq [H(\gamma) + L_C] - [H(\gamma) + L_{NL}] \nonumber \\
%     &\quad = L_C - L_{NL} = H_{CE}(\mathrm{NL}) - H_{CE}(\mathrm{Code}).
% \end{align}

% When $H_{CE}(\mathrm{Code}) < H_{CE}(\mathrm{NL})$, equivalently $L_C > L_{NL}$, the RHS is positive:
% \begin{equation}
%     \mc I(\gamma, Z_{\mathrm{Code}}) > \mc I(\gamma, Z_{\mathrm{NL}}).
% \end{equation}

% Applying \cref{thm:1}:
% \begin{align}
%     P^*_{\mathrm{NL}} - P^*_{\mathrm{Code}} \geq \frac{\mc I(\gamma, Z_{\mathrm{Code}}) - \mc I(\gamma, Z_{\mathrm{NL}})}{\log_2(K-1)} > 0.
% \end{align}

% Therefore, $P^*_{\mathrm{Code}} < P^*_{\mathrm{NL}}$.
% \end{proof}

\end{document}